\documentclass{article}

\usepackage{times}
\usepackage{natbib}

\usepackage{graphicx}
\usepackage{subfigure}

\usepackage{booktabs}
\usepackage{color}
\usepackage{algorithm}
\usepackage{algorithmic}
\usepackage{placeins}
\usepackage{amsmath}
\usepackage{amsthm}
\usepackage{amssymb}
\usepackage{natbib}
\usepackage{caption}
\usepackage{tikz}
\usepackage{tabularx}

\newtheorem{prop}{Proposition}

\newcommand{\epsv}{{\boldsymbol{\epsilon}}}
\newcommand{\muv}{{\boldsymbol{\mu}}}
\newcommand{\sigmav}{{\boldsymbol{\sigma}}}

\newcommand\numberthis{\refstepcounter{equation}\tag{\theequation}}

\usepackage[accepted]{icml2017} 
\usepackage{hyperref}
\usepackage[english]{babel}
\usepackage{url}
\usepackage{mathtools}
\usepackage{amsmath}
\usepackage{amssymb}
\usepackage{wrapfig}
\usepackage{bm}
\usepackage{psfrag}
\usepackage{tikz}
\usepackage{epstopdf}
\usepackage{adjustbox}

\newcommand{\beq}[1][\vspace{0.3em}]{#1\begin{equation}}
\newcommand{\eeq}{\end{equation}}
\newcommand{\rbx}[2][!]{
	\resizebox{#1}{!}{$#2$}	
}
\newcommand{\eqs}[1]{\begin{align}#1\end{align}}

\newcommand{\bit}{\vspace{0mm}\begin{itemize}}
\newcommand{\eit}{\vspace{0mm}\end{itemize}}
\newcommand{\ben}{\vspace{0mm}\begin{enumerate}}
\newcommand{\een}{\vspace{0mm}\end{enumerate}}

\newcommand{\Iv}[0]{{{\bf I}}}

\newcommand{\Xv}[0]{{{\bf X}}}

\newcommand{\fv}[0]{{{\bf f}}}
\newcommand{\gv}[0]{{{\bf g}}}
\newcommand{\hv}[0]{{{\bf h}}}

\newcommand{\rv}[0]{{{\bf r}}}

\newcommand{\uv}[0]{{{\bf u}}}
\newcommand{\vv}[0]{{{\bf v}}}

\newcommand{\xv}[0]{{{\bf x}}}

\newcommand{\zv}[0]{{{\bf z}}}

\newcommand{\bb}[1]{\mathbb{#1}}

\newcommand{\mc}[1]{\mathcal{#1}}

\icmltitlerunning{Learning Hierarchical Features from Generative Models}

\begin{document} 

\twocolumn[
\icmltitle{Learning Hierarchical Features from Generative Models}

\begin{icmlauthorlist}
\icmlauthor{Shengjia Zhao}{stanford}
\icmlauthor{Jiaming Song}{stanford}
\icmlauthor{Stefano Ermon}{stanford}
\end{icmlauthorlist}

\icmlaffiliation{stanford}{Stanford University}

\icmlcorrespondingauthor{Shengjia Zhao}{zhaosj12@stanford.edu}
\icmlcorrespondingauthor{Jiaming Song}{tsong@stanford.edu}
%\icmlcorrespondingauthor{Stefano Ermon}{ermon@stanford.edu}

\icmlkeywords{}

\vskip 0.3in
]
\printAffiliationsAndNotice{}

\begin{abstract}
Deep neural networks have been shown to be very successful at learning feature hierarchies in supervised learning tasks. Generative models, on the other hand, have benefited less from hierarchical models with multiple layers of latent variables. In this paper, we prove that hierarchical latent variable models do not take advantage of the hierarchical structure when trained with existing variational methods, and provide some limitations on the kind of features existing models can learn. Finally we propose an alternative architecture that do not suffer from these limitations. Our model is able to learn highly interpretable and disentangled hierarchical features on several natural image datasets with no task specific regularization or prior knowledge. 
\end{abstract}

\section{Introduction}
%One central idea of deep learning that has repeatedly proved itself useful is that of hierarchy. In particular, development of deep deterministic networks led to dramatic improvements and often human level performance in many classification, detection, translation tasks \citep{deep_architectures2009}. 

A key property of deep feed-forward networks is that they tend to learn learn increasingly abstract and invariant representations at higher levels in the hierarchy~\citep{deep_architectures2009,visualize_cnn2014}
%. This has been demonstrated by various methods of visualization for deep networks
%~\citep{visualize_cnn2014}. 
In the context of image data, low levels may learn features corresponding to edges or basic shapes, while higher levels learn more abstract features, such as object detectors~\citep{visualize_cnn2014}.

Generative models with a hierarchical structure, where there are multiple layers of latent variables, have been less successful compared to their supervised counterparts~\citep{sonderby2016ladder}. In fact, the most successful generative models often use only a single layer of latent variables \citep{deconvolutional_gan2015,conditional_pixelcnn2016}, and those that use multiple layers only show modest performance increases in quantitative metrics such as log-likelihood \citep{sonderby2016ladder,matnet_variational_network2016}. 
Because of the difficulties in evaluating generative models \citep{generative_model_evaluation2015}, and the fact that adding network layers increases the number of parameters, it is not always clear whether the improvements truly come from the choice of a hierarchical architecture. Furthermore, the capability of learning a hierarchy of increasingly complex and abstract features has only been demonstrated to a limited extent, with feature hierarchies that are not nearly as rich as the ones learned by feed-forward networks~\citep{pixel_vae2016}.

\begin{figure}[t]
%\centering
\includegraphics[width=\linewidth]{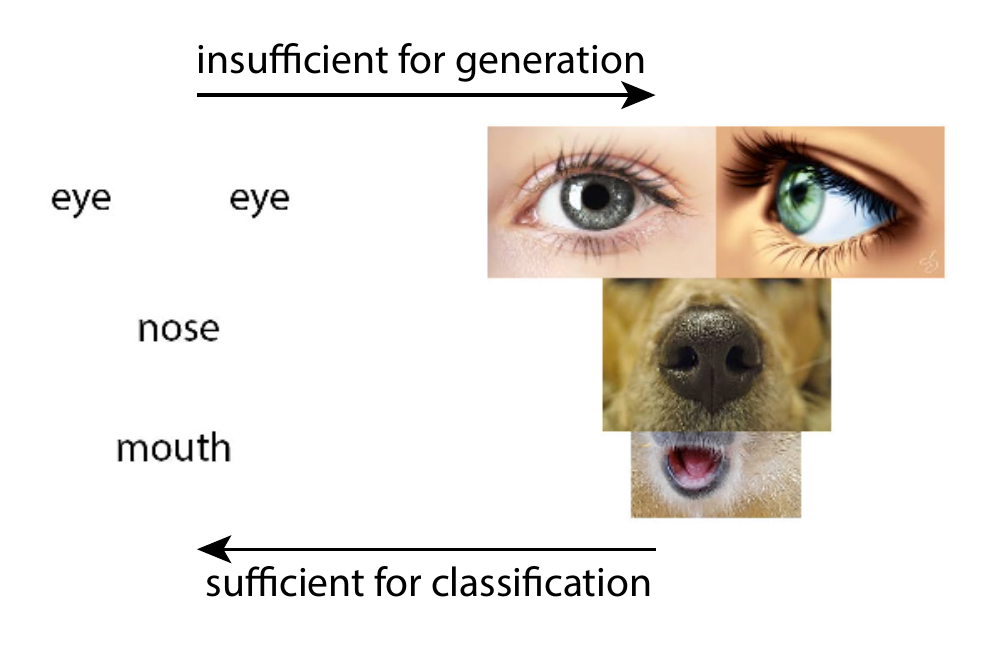}
\caption{Left: Body parts feature detectors only carry a small amount of information about an underlying image, yet, it is sufficient 
%to be clearly recognizable as a face.
for a confident classification as a face.
Right: if a hierarchical generative model attempts to reconstruct an image based on these high-level features, it could generate inconsistent images, even when each part can be perfectly generated. Even though this "face" is clearly absurd, Google cloud platform classification API can identify with 93\% confidence that this is a face. 
%(So can humans, perhaps)
%\stefano{if time, add left-right arrow "sufficient for classification", and right-left arrow "insufficient for generation"}
}
\label{fig:unreasonable_face}
\end{figure}

Part of the problem is inherent and unavoidable for any generative model. The heart of the matter is that while highly invariant and local features are often sufficient for classification, generative modeling requires preservation of details (as illustrated in Figure~\ref{fig:unreasonable_face}). In fact, most latent features in a generative model of images cannot even demonstrate scale and translation invariance.
%\jiaming{that is if we assume "eyes" to be our latent features}. 
The size and location of a sub-part often has to be dependent on the other sub-parts. For example, an eye should only be generated with the same size as the other eye, at symmetric locations with respect to the center of the face, with appropriate distance between them. The inductive biases  that are directly encoded into the architecture of convolutional networks is not sufficient in the context of generative models.

On the other hand, other problems are associated with specific models or design choices, and may be avoided with deeper understanding and careful design. The goal of this paper is to provide a deeper understanding of the design and performance of common hierarchical latent variable models. We focus on variational models, though most of the conclusions can be generalized to adversarially trained models that support inference \citep{dumoulin2016adversarially, donahue2016adversarial}. In particular, we study two classes of models with a hierarchical structure: 

\textbf{1) Stacked hierarchy:} The first type we study is characterized by recursively stacking generative models on top of each other. Most existing models \citep{sonderby2016ladder,pixel_vae2016,matnet_variational_network2016,vae_autoregressive_flow2016}, %\stefano{hierarchical, ladder, etc} actually there is few literature to cite from; every paper use HVAE as an additional feature
belong to this class. We show that these models have two limitations. 
First, 
%they have a disadvantages in representation efficiency, that is, they need more parameters to represent the same distribution compared to shallow models. This is contrary to the central motivation of using deep networks \citep{bengio2009learning} where a feed-forward deep network can represent certain functions with exponentially better efficiency. More specifically, 
we show that if these models can be trained to optimality, then the bottom layer alone contains enough information to reconstruct the data distribution, and the layers above the first one can be ignored. This result holds under fairly general conditions, and does not depend on the specific family of distributions used to define the hierarchy (e.g., Gaussian). 
Second, we argue that many of the building blocks commonly used to construct hierarchical generative models are unlikely to help us learn disentangled features. 
%We propose a whole class of such models from in a principled manner, of which current models are special cases. The principled derivations gives us insight into the properties and limitations of these models.

%\stefano{i don't like psueo hierarchy name. any other ideas? maybe horizontal hierarchy?}\shengjia{how about architectural?flat?induced?}
\textbf{2) Architectural hierarchy:} Motivated by these limitations, we turn our attention to single layer latent variable models. We propose an alternative way to learn disentangled hierarchical features by crafting a network architecture that prefers to place high-level features on certain parts of the latent code, and low-level features in others. We show that this approach, called \textbf{Variational Ladder Autoencoder}, allows us to learn very rich feature hierarchies on natural image datasets such as MNIST, SVHN ~\citep{netzer2011reading} and CelebA ~\citep{liu2015faceattributes}; in contrast, generative models with a stacked hierarchical structure fail to learn such features. 
%We hope that our work can inspire research in hierarchical generative models for structured feature learning.
%\stefano{say why it is significant}

\section{Problem Setting}
We consider a family of latent variable models specified by a joint probability distribution $p_\theta(\xv, \zv)$ over a 
set of observed variables $\xv$ and latent variables $\zv$. The family of models is assumed to be parametrized by $\theta$. 
%Let $\xv$ be the observed variables, $\zv$ be the latent variables with $p(\xv, \zv)$ as their joint distribution, 
Let $p_\theta(\xv)$ denote the marginal distribution of $\xv$. We wish to maximize the marginal log-likelihood $p(\xv)$ over a dataset $\Xv = \{\xv^{(1)}, \ldots, \xv^{(N)}\}$ drawn from some unknown underlying distribution $p_{data}(\xv)$. Formally we would like to maximize
\beq
\label{maxlik}
\log p_{\theta} (\Xv) = \sum_{n=1}^{N} \log p_\theta(\xv^{(i)})
\eeq
which is non-convex and often intractable for complex generative models, as it involves marginalization over the latent variables $\zv$. 

We are especially interested in unsupervised feature learning applications, where by maximizing (\ref{maxlik}) we hope to discover a meaningful representation for the data $\xv$ in terms of latent features given by $p_\theta(\zv|\xv)$.

\subsection{Variational Autoencoders}
A popular solution \citep{autoencoding_variational_bayes2013, variational_dbn_stochastic_bp2014} for optimizing the intractable marginal likelihood (\ref{maxlik}) is to optimize the evidence lower bound (ELBO) by introducing an inference model $q_\phi(\zv \lvert \xv)$ parametrized by $\phi$ \footnote{We omit the dependency on $\theta$ and $\phi$ for the remainder of the paper.}:
\begin{align*}
%\rbx[0.9\hsize]{
\log p(\xv) &\geq \bb{E}_{q(\zv \lvert \xv)} [\log p(\xv, \zv) - \log q(\zv \lvert \xv)] \\ 
&= \bb{E}_{q(\zv | \xv)} [\log p(\xv | \zv)] - \bb{KL}(q(\zv \lvert \xv) \lVert p(\zv)) \\
&= \mc{L}(\xv; \theta, \phi) \numberthis \label{eq:vae-obj}
%}
\end{align*}
where $\bb{KL}$ is the Kullback-Leibler divergence.
%\stefano{kl undefined}

%   \aditya{add $\theta$, $\phi$ as subscripts in all the above equations}
%To sample from $p(\xv, \zv)$, $\zv$ is first sampled from the prior $p(\zv)$, and then transformed into $\xv$ by sampling from the conditional $p(\xv| \zv)$. Hence, it is intuitively desirable to have the posterior $q(\zv \lvert \xv)$ match $p(\zv)$, which is highlighted by the following  equivalent formulation of $\mc{L}(\xv; \theta, \phi)$:
%\beq
%\mc{L}(\xv; \theta, \phi) = \bb{E}_{q(\zv | \xv)} [\log p(\xv | \zv)] - \bb{KL}(q(\zv \lvert \xv) \lVert p(\zv))
%\label{eq:vae-obj}
%\eeq
%where the first term on the right hand side is a \textit{reconstruction loss} for $\xv$.

%\stefano{is this unnencessary?}
%\begin{align*}
%\mathcal{L}_1 = E_{q_\phi(x, \zv)}[\log p_\theta(x|\zv)] - D_1(q_\phi(\zv)||p_\thetav(\zv)) \numberthis \label{equ:single_layer_vae_noreg}
%\end{align*}

\subsection{Hierarchical Variational Autoencoders}
A hierarchical VAE (HVAE) can be thought of as a series of VAEs stacked on top of each other. It has the following hierarchy of latent variables $\zv = \{\zv_1, \ldots, \zv_L\}$, in addition to the observed variables $\xv$. We use the notation convention that $\zv_1$ represents the lowest layer closest to $\xv$ and $\zv_L$ the top layer.
%\stefano{where $\ell$ indicates depth? which one is shallow/deep}. 
Using chain rule, the joint distribution $p(\xv, \zv_1, \ldots, \zv_L)$ can be factored as follows
\beq
p(\xv, \zv_1, \ldots, \zv_L) = p(\xv \lvert \zv_{>0}) \prod_{\ell=1}^{L-1} p(\zv_\ell \lvert \zv_{>\ell}) p(\zv_{L}) \label{equ:hvae_generator}
\eeq
where $\zv_{>\ell}$ indicates $(\zv_{\ell+1}, \cdots, \zv_L)$, and $\zv_{>0} = \zv = (\zv_1, \ldots, \zv_L)$. Note that this factorization via chain-rule
%only assume the natural auto-regressive structure and 
is fully general. In particular it accounts for recent models that use shortcut connections~\citep{vae_autoregressive_flow2016, matnet_variational_network2016}, where each hidden layer $\zv_\ell$ directly depends on all layers above it ($\zv_{>\ell}$). We shall refer to this fully general formulation as autoregressive HVAE.

Several models assume a Markov independence structure on the hidden variables, leading to the following simpler factorization
%use a more traditional architecture, where $p$ takes a Markov form
\citep{variational_dbn_stochastic_bp2014, pixel_vae2016, ladder_variational_network2015}
\begin{align*}
p(\xv, \zv) = p(\xv|\zv_\ell) \prod_{l=1}^{L-1} p(\zv_\ell|\zv_{\ell+1}) p(\zv_L) \numberthis \label{equ:factorized_hvae}
\end{align*}
We refer to this common but more restrictive formulation as Markov HVAE. 

For the inference distribution $q(\zv|\xv)$ we do not assume any factorized structure to account for complex inference techniques used in recent work
%\aditya{what 'advances'? describe a bit more in detail} 
\citep{ladder_variational_network2015, matnet_variational_network2016}. We also denote $q(\xv, \zv) = p_{data}(\xv) q(\zv|\xv)$.
%If we assume $q(\zv|\xv)$ is also hierarchical: 
%\beq
%q(\zv \lvert \xv) = q(\zv_1 | \xv)\prod_{\ell=1}^{L-1} q(\zv_{\ell+1} | \zv_\ell) \\, \label{eq:hvae-inference}
%\eeq
%the objective would then become

%\stefano{$p_{data}(\xv)$ is undefined. define it somewhere above near eq 1}
% \begin{adjustbox}{minipage=1.1\linewidth,scale=0.9}
% \vspace{-1em}
% \begin{gather}
% 	\bb{E}_{q(\zv \lvert \xv)}[\log p(\zv_L) + \sum_{\ell=1}^{L} \log p(\zv_{\ell-1} | \zv_{\ell}) - \log q(\zv \lvert \xv)] \nonumber \\
% 	= \bb{E}_{q(\zv \lvert \xv)} [\log p(\xv \lvert \zv_1)] - \sum_{\ell=1}^{L} \bb{KL}(q(\zv_\ell | \zv_{\ell-1}) \lVert p(\zv_\ell | \zv_{\ell+1})) \label{eq:hvae-obj}
% \end{gather}
% \end{adjustbox}

%\stefano{for completeness, this should be connected with the previous elbo expression}
Both $p(\xv|\zv)$ and $q(\zv|\xv)$ are jointly optimized, as before in Equation (\ref{eq:vae-obj}), to maximize the ELBO objective
\begin{align*}
\mathcal{L}_{ELBO} &= E_{p_{data}(\xv)} E_{q(\zv|\xv)}[\log p(\xv|\zv)] - \\ 
& \qquad \qquad E_{p_{data}(\xv)}[\bb{KL}(q(\zv|\xv)||p(\zv))] \\ 
%&= E_{q(\zv, \xv)}\left[\log \frac{\prod_{l=0}^{L} p(\zv_l|\zv_{>l})}{q(\zv|x)}\right] \\ 
%&= \bb{E}_{q(\zv, \xv)} [\log p(\xv \lvert \zv)] - \sum_{\ell=1}^{L} KL(q(\zv_{>\ell} | \zv_{\ell-1}) \lVert p(\zv_\ell | \zv_{\ell+1})) \numberthis \label{eq:hvae-obj} \\
&=\sum_{l=0}^{L} E_{q(\zv, \xv)}[\log p(\zv_l|\zv_{>l})] + H(q(\zv|\xv))  \numberthis \label{equ:hvae_criteria}
\end{align*}
where we define $\zv_0 \equiv \xv$, $\zv_{L+1} \equiv {\bf 0}$, and $H$ the entropy of a distribution, and expectation over $p_{data}(x)$ is estimated by the samples in the dataset. This can be interpreted as stacking VAEs on top of each other.

%\stefano{entropy undefined. kl notation changed from above}

% For Markov HVAE this can be written as 
% \begin{align*}
% \mathcal{L}_{ELBO} &= \sum_{l=0}^{L} E_{q(\zv, \xv)}[\log p(\zv_l|\zv_{l+1})] + H(q(\zv|\xv))
% \end{align*}
%&= \bb{E}_{q(\zv, \xv)} [\log p(\xv \lvert \zv)] - \sum_{\ell=1}^{L} KL(q(\zv_{>\ell} | \zv_{\ell-1}) \lVert p(\zv_\ell | \zv_{\ell+1})) \numberthis \label{eq:hvae-obj} \\

\section{Limitations of Hierarchical VAEs}
\label{sec:limitations}
\subsection{Representational Efficiency}
One of the main reasons deep hierarchical networks are widely used as function approximators is their representational power. It is well known that certain functions can be represented much more compactly with deep networks, requiring exponentially less parameters compared to shallow networks~\citep{bengio2009learning}. However, we show that under ideal optimization of $\mathcal{L}_{ELBO}$, HVAE models do not lead to improved representational power. This is because for a well trained HVAE, a Gibbs chain on the bottom layer, which is a single layer model, can be used to recover $p_{data}(\xv)$ exactly. 

We first show this formally for Markov HVAE with the following proposition
\begin{prop}
\label{prop:hvae_redundency_markov}
$\mathcal{L}_{ELBO}$ in Eq.(\ref{equ:hvae_criteria}) is globally maximized as a function of $q(\zv|\xv)$ and $p(\xv|\zv)$ when $\mathcal{L}_{ELBO} = -H(p_{data}(x))$. If $\mathcal{L}_{ELBO}$ is globally maximized for a Markov HVAE, the following Gibbs sampling chain converges to $p_{data}(\xv)$ if it is ergodic
\begin{align*}
\zv_1^{(t)} &\sim q(\zv_1|\xv^{(t)}) \\  
\xv^{(t+1)} &\sim p(\xv|\zv_1^{(t)}) \numberthis \label{equ:gibbs_chain_factorized}
\end{align*}
\end{prop}

%In fact, it represents distributions less efficiently compared to a single layer model.
%\stefano{restate an prove in terms of the factored model below (5)}

\begin{proof}[Proof of Proposition~\ref{prop:hvae_redundency_markov}]
We notice that 
\begin{align*}
& \mathcal{L}_{ELBO} = E_{p_{data}(\xv) q(\zv|\xv)} \left[ \log \frac{p(\xv, \zv)}{q(\zv|\xv)} \right] \\ 
&\qquad = E_{p_{data}(\xv)q(\zv|\xv)}\left[ \log \frac{p(\zv|\xv)}{q(\zv|\xv)} \right] + E_{p_{data}(\xv)} [\log p(\xv)] \\
& \qquad = -E_{p_{data}(\xv)}[\bb{KL}(q(\zv|\xv)||p(\zv|\xv))] \\ 
& \qquad  \qquad - \bb{KL}(p_{data}(\xv)||p(\xv)) - H(p_{data}(\xv))
\end{align*}
%\stefano{add some derivation of this}
By non-negativity of KL-divergence, and the fact that KL divergence is zero if an only if the two distributions are identical, it can be seen that this is uniquely optimized when $p(\xv) = \int_\zv p(\xv, \zv) d\zv = p_{data}(\xv)$ and $\forall \xv, q(\zv|\xv) = p(\zv|\xv)$ and the optimum is
\[ \mathcal{L}^*_{ELBO} = -H(p_{data}(\xv)) \]
This also implies that $\forall \xv$
\begin{align*}
q(\xv|\zv_1) = \frac{q(\zv_1|\xv)p_{data}(\xv)}{q(\zv_1)} = p(\xv|\zv_{1})\numberthis \label{equ:factorized_hvae2}
\end{align*}
Because the following Gibbs chain converges to $p_{data}(\xv)$ when it is ergodic
\begin{align*}
\zv_1^{(t)} &\sim q(\zv_1|\xv^{(t)}) \\ 
\xv^{(t+1)} &\sim q(\xv|\zv_1^{(t)}) 
\end{align*}
We can replace 
$q(\xv|\zv_1^{(t)})$ with $p(\xv|\zv_1^{(t)})$ using (\ref{equ:factorized_hvae2}) and the chain still converges to $p_{data}(\xv)$.
\end{proof}
%This means that recursive hierarchies actually use network capacity less efficiently than a single layer model. 
Therefore under the assumptions of Proposition 1 we can sample from $p_{data}(\xv)$ without using the latent code $(\zv_2, \cdots, \zv_L)$ at all. Hence, optimization of the $\mathcal{L}_{ELBO}$ objective and efficient representation are conflicting, in the sense that optimality implies some level of redundancy in the representation. 

We demonstrate that this phenomenon occurs in practice, even though the conditions of Proposition 1 might not be met exactly. We train a factorized three layer VAE in Equation (\ref{equ:factorized_hvae}) on MNIST by optimizing the ELBO criteria Equation (\ref{equ:hvae_criteria}). We use a model where each conditional distribution is factorized Gaussian $p(\zv_\ell|\zv_{\ell+1}) = \mathcal{N}(\mu_\ell(\zv_{\ell+1}), \sigma_\ell(\zv_{\ell+1}))$ where $\mu_\ell$ and $\sigma_\ell$ are deep networks. %where \stefano{give more details on architecture used here}, 
We compare: the samples generated by the Gibbs chain in Equation (\ref{equ:gibbs_chain_factorized}) with samples generated by ancestral sampling with the entire model in Figure~\ref{fig:recursive_mc}. We observe that the Gibbs chain generates samples (left panel) with similar visual quality as ancestral sampling with the entire model (right panel), even though {\em the Gibbs chain only used the bottom layer of the model}. % \stefano{discuss the results shown in the figure}

This problem can be generalized to autoregressive HVAEs. %\aditya{can you prove this?}
% , as we show in the following proposition
One can sample from $p_{data}(\xv)$ without using $p(\zv_\ell|\zv_{>\ell}), 1 \leq \ell < L$ at all. We prove this in the Appendix.
%All the dependency we modeled in the high layers did not give us any advantage in modeling the data distribution. 
  %This breaks the hope that recursive hierarchies might give a more efficient representation of data compared to a single layer model. We have represented the distribution without using higher layer latent codes at all.

\subsection{Feature learning}
Another significant advantage of hierarchical models for supervised learning is that they learn rich and disentangled hierarchies of features. This has been demonstrated for example using various visualization techniques \citep{visualize_cnn2014}. However, we show in this section that typical HVAEs do not enjoy this property. %This section extends the arguments in \citep{zhao2017vae} to the hierarchical case.

Recall that we think of $p(\zv|\xv)$ as a (probabilistic) feature detector, and $q(\zv|\xv)$ as an approximation to $p(\zv|\xv)$. It might therefore be natural to think that $q$ might learn hierarchical features similarly to a feed-forward network $x \to \zv_\ell \to \cdots \to \zv_L$, where higher layers correspond to higher level features that become increasingly abstract and invariant to nuisance variations. However if $q(\zv_{>\ell}|\zv_\ell)$ maps low level features to high level features, then the reverse mapping $q(\zv_\ell|\zv_{>\ell})$ maps high level features to likely low level sub-features. For example, if $\zv_L$ correspond to object classes, then $q(\zv_{L-1}|\zv_L)$ could represent the distribution over object subparts given the object class.  
%The difficulty is that the generative model used to define the HVAE poses some restrictions as to what this distribution can be. \stefano{not really true? q can be anything in principle. it's more an issue with the learning objective}

%\stefano{i don't quite like the next two paragraphs. this entire section  feels more handwavy that it can be}
Suppose we train $\mathcal{L}_{ELBO}$ in Equation (\ref{equ:hvae_criteria}) to optimality, we would have
\[ p(\xv) = p_{data}(\xv), q(\zv|\xv) = p(\zv|\xv) \]
Recall that
\begin{align*} 
q(\xv, \zv) &:= p_{data}(\xv)q(\zv|\xv)  \\ 
p(\xv, \zv) &:= p(\zv)p(\xv|\zv) = p(\xv) p(\zv|\xv) 
\end{align*}
Comparing the two we see that
\[ p(\xv, \zv) = q(\xv, \zv) \]
if the joint distributions are identical, then any conditional distribution would also be identical, which implies that for any $\zv_{>\ell}$, $ q(\zv_\ell|\zv_{>\ell}) = p(\zv_\ell|\zv_{>\ell}) $. 
%and that we are able to match the data distribution exactly. If we have a sufficiently expressive family of inference distributions $q(\zv|\xv)$, 
%which implies $q(\zv|\xv) =p(\zv|\xv)$, 
%i.e., the approximate posterior equals the true posterior. Along with $p(\xv) = p_{data}(\xv)$, we have $p(\xv, \zv) = q(\xv, \zv)$
%As previously observed in \citep{zhao17vae} for the case of single layer VAEs, this 

% \begin{prop}
% \label{prop:hvae_posterior}
% If we train $\mathcal{L}_{ELBO}$ in Equation (\ref{equ:hvae_criteria}) to optimality for sufficiently rich distribution families $q$, and if $p(\xv, \zv)$ has marginal $p_{data}(\xv)$. Then for all $0 \leq \ell < L$, we have for all $\zv_{>\ell}$
% \[ p(\zv_\ell|\zv_{>\ell}) = q(\zv_\ell|\zv_{>\ell}) \]
% \end{prop}
% \begin{proof}[Proof of Proposition~\ref{prop:hvae_posterior}]
% We can rewrite the ELBO bound as
% \[ \mathcal{L}_{ELBO} = E_{p_{data}((\xv)}[\log p(\xv)] - KL(q(\zv|\xv)||p(\zv|\xv)) \]
% which is optimized if $q(\zv|\xv) = p(\zv|\xv)$, but since $p(\xv) = p_{data}(\xv) = q(\xv)$, then $q(\xv, \zv) = p(\xv, \zv)$, which gives us equality of all conditionals.
% \end{proof}

For the majority of models the conditional distributions $p(\zv_\ell|\zv_{>\ell})$ belong to a very simple distribution family such as parameterized Gaussians \citep{autoencoding_variational_bayes2013} \citep{variational_dbn_stochastic_bp2014} \citep{ladder_variational_network2015} \citep{vae_autoregressive_flow2016}. Therefore for a perfectly optimized $\mathcal{L}_{ELBO}$ in the Gaussian case, the only type of feature hierarchy we can hope to learn is one under which $q(\zv_\ell|\zv_{>\ell})$ is also Gaussian. This limits the hierarchical representation we can learn. In fact, the hierarchies we observe for feed-forward models \citep{visualize_cnn2014} require complex multimodal distributions to be captured. For example, the distribution over object subparts for an object category is unlikely to be unimodal and cannot be well approximated with a Gaussian distribution. 

More generally, as shown in \citep{zhao2017vae}, even when $\mathcal{L}_{ELBO}$ is not globally optimized, optimizing $\mathcal{L}_{ELBO}$ encourages $q(\zv_\ell|\zv_{>\ell})$ and $p(\zv_\ell|\zv_{>\ell})$ to match. Because $p(\zv_\ell|\zv_{>\ell})$ belong to some distribution family, such as Gaussians. This encourages $q(\zv_\ell|\zv_{>\ell})$ to belong to that distribution family as well.
%should be fairly complex, or the distribution over edge positions for given detected shapes. The model cannot capture these relationships with the Gaussian constraint.
%\stefano{maybe we should say that even if there is no equality, the elbo objective is still doing m-projections. basically, convice people that the argument is robust and to some extent hold even when thm conditions are not met}

We experimentally demonstrate these intuitions in Figure~\ref{fig:hvae_hierarchy}, where we train a three layer Markov HVAE with factorized Gaussian conditionals $p(\zv_\ell|\zv_{\ell+1})$ on MNIST and SVHN. Details about the experimental setup are explained in the Appendix. 
%\stefano{add details on what is trained. what objective? on what data? say more details are provided in the appendix. actually add the experimental details there. right now it's not present}
%\stefano{add details}. 
As suggested in \citep{autoencoding_variational_bayes2013}, we reparameterize the stochasticity in $p(\zv_\ell|\zv_{\ell+1})$ using a separate noise variable $\epsv_\ell \sim \mathcal{N}(0, I)$, and implicitly rewrite the original conditional distribution as 
\[ \zv_\ell = \mu_\ell(\zv_{\ell+1}) + \sigma_\ell(\zv_{\ell+1}) \odot \epsv_\ell \]
where $\odot$ indicates element-wise product. We fix the value of $\epsv_k$ to a random sample from $\mathcal{N}(0, I)$ at all layers $k=1, \cdots, \ell-1, \ell+1, \cdots, L$ except for one, and observe the variations in $\xv$ generated by randomly sampling $\epsv_\ell$. We observe in Figure~\ref{fig:hvae_hierarchy} that only very minor variations correspond to lower layers (Left and center panels), and almost all the variation is represented by the top layer (Right panel). %\stefano{explain exactly what you are doing. fixing some parts, and sampling from others?} 
More importantly, no notable hierarchical relationship between features is observed. 
%M-projections of $p$ onto $q$. if $p$ is Gaussian, it will learn Gaussians, and intutively it's unlikely to work
%\section{Flat VAE}

%\stefano{i would cut this paragraph. does not make much sense on its own}
%On the other hand, if we use a complex conditional distribution $p(\zv_\ell|\zv_{>\ell})$, then \citet{lossy_vae2016} show that optimization of $\mathcal{L}_{ELBO}$ encourages the model to ignore latent variables altogether. It is in general not known how to select conditional distribution families $p(\zv_\ell|\zv_{>\ell})$ which lead to meaningful feature hierarchies, even though promising directions have been proposed \citep{zhao2017vae}.
%$q$ can learn meaningful feature hierarchies.

\begin{figure*}
\centering
\begin{subfigure}
\centering
\includegraphics[height=0.1\textwidth]{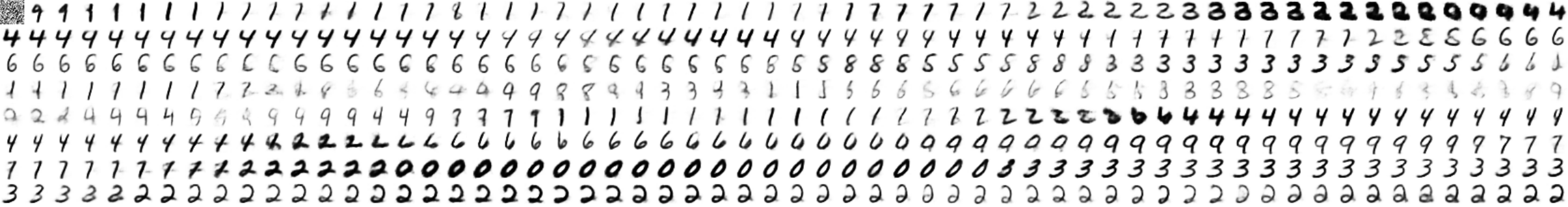}
\end{subfigure}
\begin{subfigure}
\centering
\includegraphics[height=0.1\textwidth]{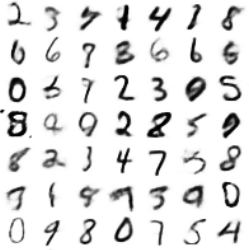}
\end{subfigure}
\caption{\textbf{Left:} Samples obtained by running the Gibbs sampling chain in Proposition 1, using only the bottom layer of a 3-layer recursive hierarchical VAE. \textbf{Right:} samples generated by ancestral sampling from the same model.
The quality of the samples is comparable, indicating that the bottom layer contains enough information to reconstruct the data distribution.
}
\label{fig:recursive_mc}
\end{figure*}

\begin{figure*}[h]
\centering
\begin{subfigure}
\centering
\includegraphics[trim=55 0 0 0, clip, width=0.9\linewidth]{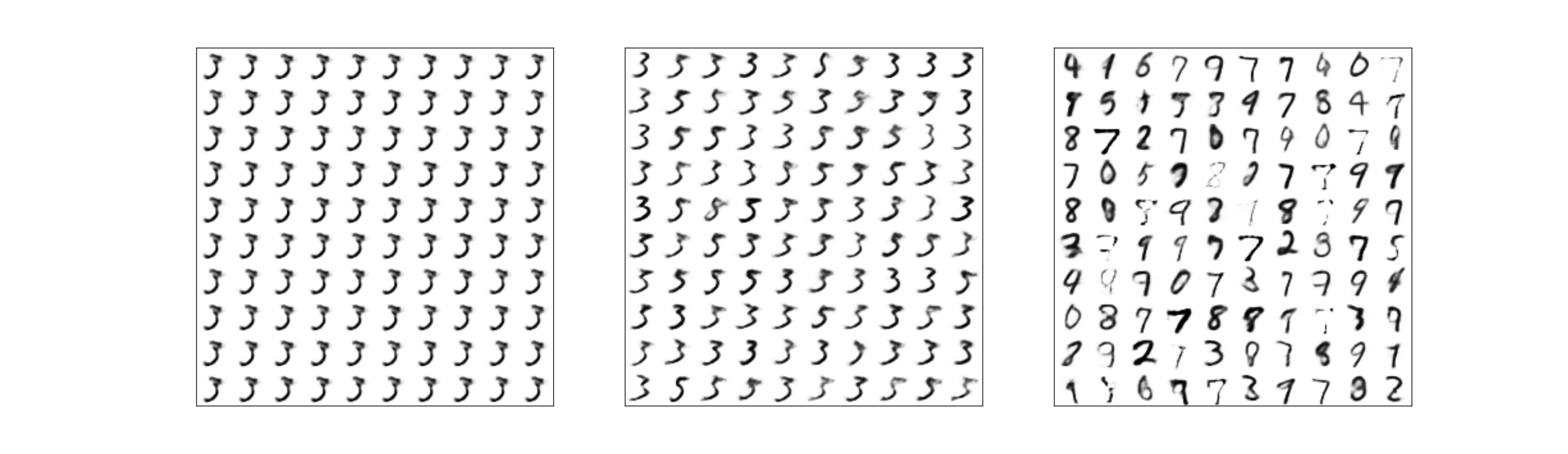}
\end{subfigure}
\begin{subfigure}
\centering
\includegraphics[width=0.75\linewidth]{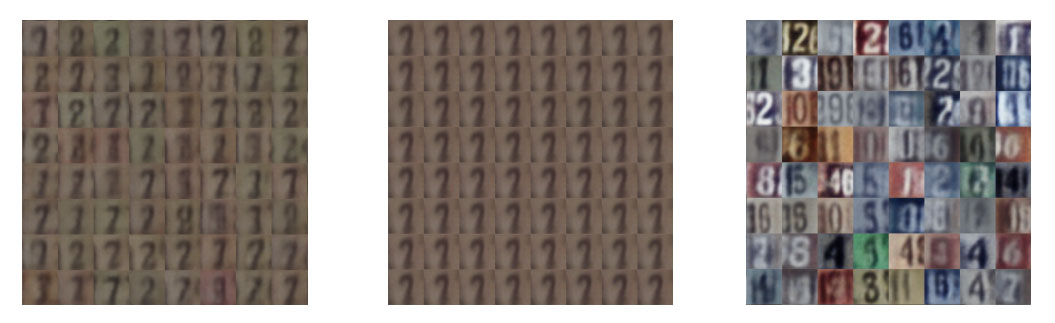}
\end{subfigure}
\caption{A hierarchical three layer VAE with Gaussian conditional distributions $p(\zv_l|\zv_{l+1})$ does not learn a meaningful feature hierarchy on MNIST and SVHN when trained with the ELBO objective.
\textbf{Left panel:} Samples generated by sampling noise $\epsv_1$ at the bottom layer, while holding $\epsv_2$ and $\epsv_3$ constant.
\textbf{Center panel:} Samples generated by sampling noise $\epsv_2$ at the middle layer, while holding $\epsv_1$ and $\epsv_3$ constant.
\textbf{Right panel:} Samples generated by sampling noise $\epsv_3$ at the top layer, while holding $\epsv_1$ and $\epsv_2$ constant.
%Plotted are the samples generated when we fix reparameterized noise variable for all layers except one. From left to right the randomly sampled layer go from bottom layer to top layer. 
For both MNIST and SVHN we observe that the top layer represents essentially all the variation in the data (right panel), leaving only very minor local variations for the lower layers (left and center panels). Compare this with the rich hierarchy learned by our VLAE model, shown in Figures~\ref{fig:mnist_ladder_separate_feature} and \ref{fig:ladder_vae_svhn2}.}
\label{fig:hvae_hierarchy}
\end{figure*}

\section{Variational Ladder Autoencoders}
\label{sec:vlae}
% given that we don't know how to choose conditional $p$ in a way that captures the invariance of features, e.g., in natural images, we explore a flat architecture.
Given the limitations of hierarchical architectures described in the previous section, we focus on an alternative approach to learn a hierarchy of disentangled features.

%The recursive approach learns $q(z_{>l})p(x, z_{\leq l}|z_{>l})$ and stacks hierarchical layers to approximate the possibly intractable $q(z_{>l})$. However this is not the only way a hierarchy can be learned. 
%Given that we have little knowledge over how to select a suitable conditional $p$ in a way that captures the invariance of features, such as in natural images, learning structured representations through a hierarchical variational autoencoder is difficult. 

Our approach is to define a simple distribution with no hierarchical structure over the latent variables $p(\zv) = p(\zv_1, \cdots, \zv_L)$. For example, the joint distribution $p(\zv)$ can be a white Gaussian. Instead we encourage the latent code $\zv_1, \cdots, \zv_L$ to learn features with different levels of abstraction by carefully choosing the mappings $p(\xv|\zv)$ and $q(\zv|\xv)$ between input $\xv$ and latent code $\zv$. Our approach is based on the following intuition: 

%\stefano{explain this better. is the mapping $p$ or $q$, or both?}
%\textbf{Assumption:} More abstract features require more expressive inference and generative mapping between input and the features.

\textbf{Assumption:} If $\zv_i$ is more abstract than $\zv_j$, then the inference mapping $q(\zv_i|\xv)$ and generative mapping when other layers are fixed $p(\xv|\zv_i, \zv_{\neg i}=z^0_{\neg_i})$ requires a more expressive network to capture.

%\stefano{can you write it in a slightly more formal way? more complex family or more complex mapping?}

%\stefano{this paragraph is confusing. low-level could be depth or level of abstraction. try to rewrite and make it clearer }
%\stefano{what's confusing is the relationship among the features.}
%\stefano{try to be more specific. not just ``relate''}
This informal assumption suggests that we should use neural networks of different level of expressiveness to generate the corresponding features; the more abstract features require more expressive networks, and vice versa. We loosely quantify expressiveness with depth of the network. Based on these assumptions we are able to design an architecture that disentangles hierarchical features for many natural image datasets. 

%We note that this architecture is only one of the many possibilities in the flat approach, and that more complex dependencies could be learned if one selects the correct architecture.

% \subsection{Assumptions}
% \label{sec:assumptions}
% Our architecture is based on two assumptions: 
% \begin{enumerate}
% \item There are meaningful independent latent factors of variation.
% \item 
% \end{enumerate}

% The first assumption ensures that latent code can be effectively converted into subparts that have certain semantic meaning, which is the basis of learning hierarchical features. 
%, and low-level features are represented by shallower networks.

% \subsection{Sampling as Noise Injection}
% Before we delve into the details of our architecture, we see that the generative process of HVAE at every level is equivalent to: 
% \beq
% \zv_\ell = \mu(\zv_{\ell+1}) + \sigma(\zv_{\ell+1}) \cdot \ev_\ell
% \eeq
% where the first term of the right hand side is a deterministic mapping from $\zv_{t+1}$, and $\ev_\ell \sim \mc{N}(0, \Iv)$ on the second term is the injected noise, which is independent of the other variables. 

\subsection{Model Definition}

% \begin{figure}
% \centering
% \begin{subfigure}
% \centering
% \includegraphics[height=0.5\linewidth]{plot/hvae_small}
% \end{subfigure}
% ~ \quad ~
% \begin{subfigure}
% \centering
% \includegraphics[height=0.5\linewidth]{plot/vlae_small}
% \end{subfigure}
% \caption{An illustration for HVAE (left) and VLAE (right). The most notable difference between these two methods is that HVAE assumes dependency between the latent variables, whereas VLAE does not. The independence property of VLAE allows learning disentangled latent variables at different levels.} %\aditya{isn't $z_1$ dep on $z_2$ in VLAE as per Eq.(8)?}
% \label{fig:hvae-vlae}
% \end{figure}

%We see that for HVAE, noise are injected at different levels of the neural network, whereas the higher level noise is propagated through a deeper network, and lower level noise is propagated through a shallower network. If we treat noise as our latent variables, this directly satisfies the two assumptions presented in Section \ref{sec:assumptions}. Therefore, we consider the following model as our generative model. 
%According to our assumption, any model that utilizes deep mappings for higher-level latent codes and shallow ones for lower-level latent codes would suffice. 

We decompose the latent code into subparts $\zv = \{\zv_1, \zv_2, \ldots \}$, where $\zv_1$ relates to $\xv$ with a shallow network, and increase network depth up to $\zv_L$, which relates to $\xv$ with a deep network. 
%Despite its simplicity this strategy is very effective in learning of disentangled latent hierarchies in our experiments.
In particular, we share parameters with a ladder-like architecture \citep{ladder_network2015, deconstructing_ladder2015}. Because of this similarity we denote this architecture as Variational Ladder Autoencoder (VLAE). Formally, our model, shown in Figure~\ref{fig:vael} is defined as follows

\textbf{1) Generative Network:} $p(\zv) = p(\zv_1, \cdots, \zv_L)$ is a simple prior on all latent variables. We choose it as a standard Gaussian $\mathcal{N}(0, I)$. The conditional distribution $p(\xv|\zv_1, \zv_2, \ldots, \zv_L)$ is defined implicitly as:
\eqs{
\tilde{\zv}_{L} &= \fv_L(\zv_L) \\
\tilde{\zv}_\ell &= \fv_\ell(\tilde{\zv}_{\ell+1}, \zv_\ell) \ \ \ \ell = 1, \cdots , L-1 \\
\xv &\sim \rv(\xv; \fv_0(\tilde{\zv}_1))  \label{eq:vlae_generator}
}
%\stefano{use boldface for the $f$}
%\stefano{what is the prior on $z$??}
%\stefano{also move pictures up on page 7 if possible}
%\stefano{i would spend another 30 min trying to make this description clearer. i doubt anybody would understand how the model works or be able to reimplement based on this descriptio}
%\stefano{use a different symbol for $p$?}
%\stefano{last equation seems wrong}
%\stefano{is this a procedure to sample from $x$, or to define the density??}
where $\fv_\ell$ is parametrized as a neural network, and $\tilde{\zv}_\ell$ is an auxiliary variable we use to simplify the notation. $\rv$ is a distribution family parameterized by $\fv_0(\tilde{\zv}_1)$. In our experiments we use the following choice for $\fv_\ell$:
\eqs{
	\tilde{\zv}_\ell &= \uv_\ell([\tilde{\zv}_{\ell+1}; \vv_\ell(\zv_{\ell})])
	%\tilde{\zv}_\ell &= o_\ell([\tilde{\zv}_{\ell+1}; \tilde{\zv_\ell}])
}
where $[\cdot; \cdot]$ denotes concatenation of two vectors, and $\vv_\ell, \uv_\ell$ are neural networks. We choose $\rv$ as a fixed variance factored Gaussian with mean given by $\muv_r = \fv_0(\tilde{\zv}_1)$.

\textbf{2) Inference Network:} For the inference network, we choose $q(\zv \lvert \xv)$ as
\eqs{
	\hv_\ell &= \gv_\ell(\hv_{\ell-1}) \\
	\zv_\ell &\sim \mathcal{N}(\muv_\ell(\hv_\ell), \sigmav_\ell(\hv_\ell))
}
where $\ell = 1, \cdots, L$, $\gv_\ell$, $\muv_\ell$, $\sigmav_\ell$ are neural networks, and $\hv_0 \equiv \xv$.
%\stefano{this is mixing up the distribution with the implicit definition..}
%\stefano{where does $x$ come in?}

\textbf{3) Learning:} For learning we use the ELBO criteria as in Equ.(\ref{eq:vae-obj}):
\beq
\rbx[0.88\hsize]{
\mc{L}(\xv) = \bb{E}_{q(\zv | \xv)} [\log p(\xv | \zv)] - \bb{KL}(q(\zv \lvert \xv) \lVert p(\zv))
} \label{eq:vael-obj}
\eeq
where $p(\zv) = \mc{N}(0, \Iv)$ denotes the prior for $\zv$. This is tractable if $\rv$ has tractable log likelihood, i.e. when $\rv$ is a Gaussian. 

This is essentially the inference and learning framework for a one-layer VAE; the hierarchy is only implicitly defined by the network architecture, therefore we call this model \textbf{flat hierarchy}. Motivated by our earlier theoretical results, we do not use additional layers of latent variables.

%\stefano{move the discussion on why it can be done efficiently here}

\subsection{Comparison with Ladder Variational Autoencoders}
\begin{figure}
\centering
\begin{subfigure}
\centering
\includegraphics[height=0.5\linewidth]{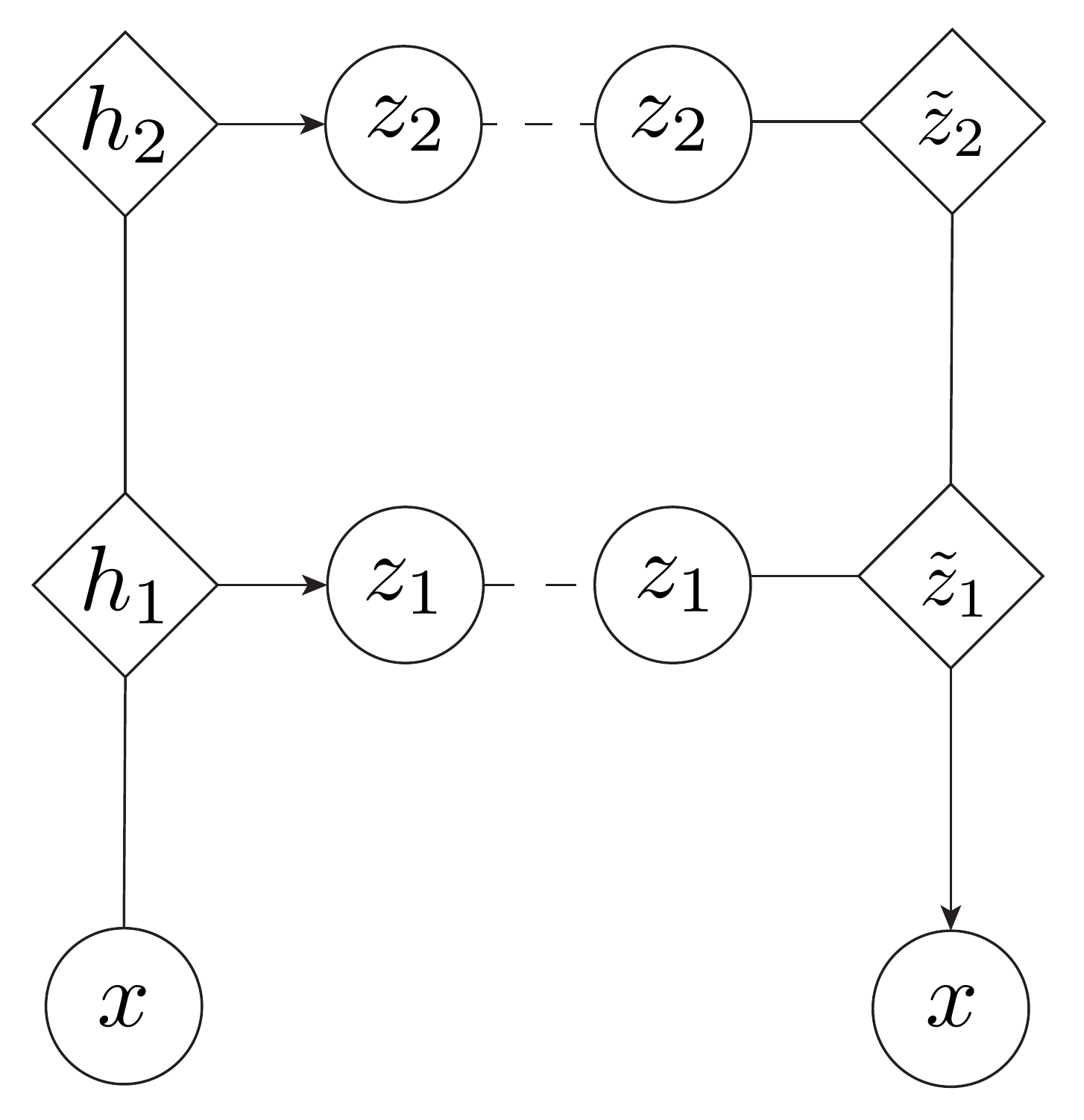}
\end{subfigure}
~
\begin{subfigure}
\centering
\includegraphics[height=0.5\linewidth]{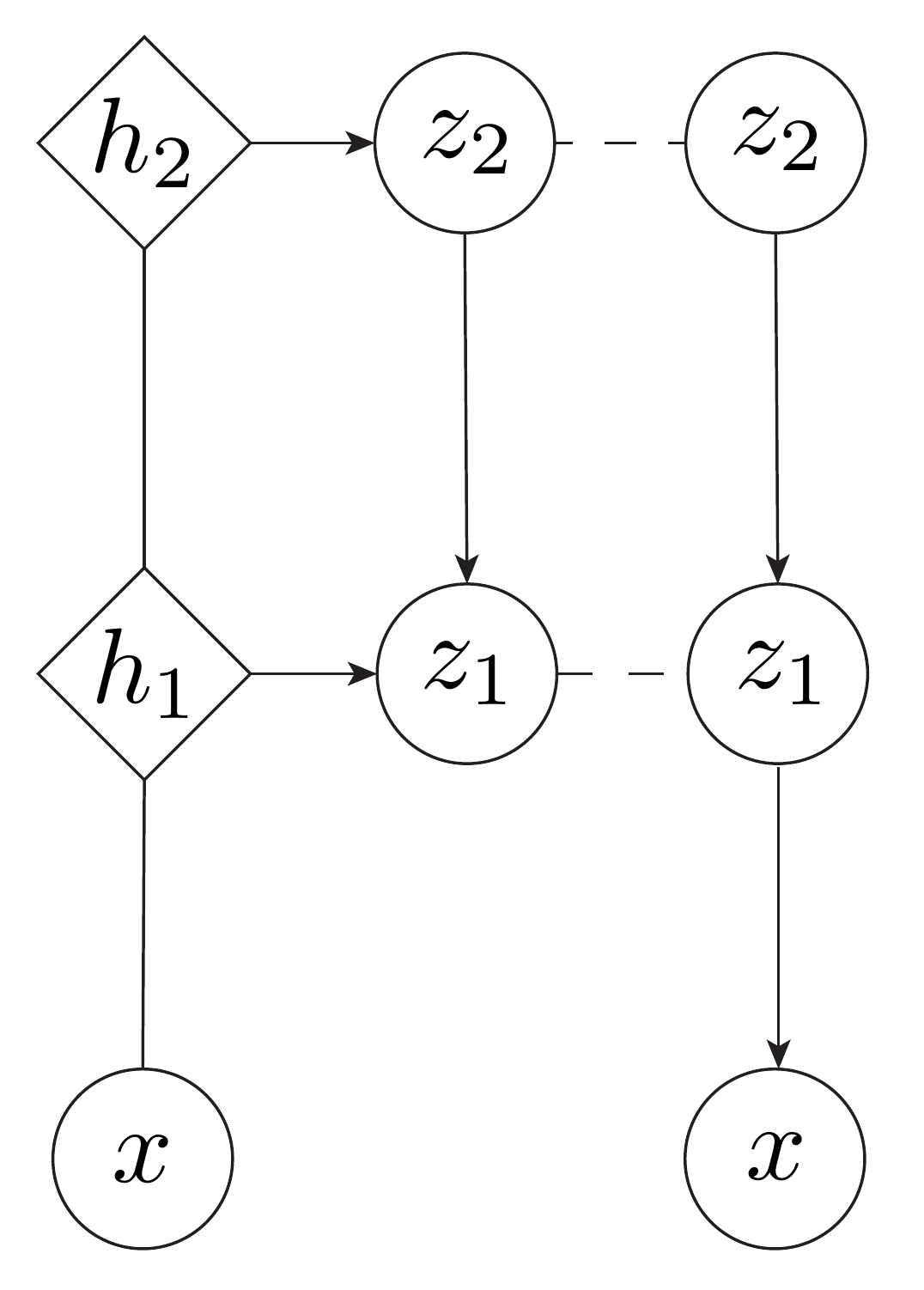}
\end{subfigure}
\caption{Inference and generative models for VLAE (left) and LVAE (right). Circles indicate stochastic nodes, and squares are deterministically computed nodes. Solid lines with arrows denote conditional probabilities; solid lines without arrows denote deterministic mappings; dash lines indicates regularization to match the prior $p(\zv)$. Note that in VLAE, we do not attempt to regularize the distance between $\hv$ and $\tilde{\zv}$. 
} %\stefano{explain dash vs solid line}

\label{fig:vael}
\end{figure}

%\stefano{this part should go in the learning section}

%The term ``ladder'' comes from the fact that our network utilizes ladder connections ~\citep{ladder_network2015}. The term ``variational'' comes from the fact that we match the inference distribution with the prior distribution, by minimizing KL divergence in the ELBO. Our architecture can be treated as a variational form of ladder autoencoders.

Our architecture resembles the ladder variational autoencoder (LVAE) \citep{sonderby2016ladder}. However the two models are very different. The purpose of our architecture is to connect subparts of the latent code with networks of different expressive power (depth); 
%under the same bandwidth for latent code, 
the model is encouraged to place high-level, complex features at the top, and low-level, simple features at the bottom, in order to reach lower reconstruction error with latent codes of the same capacity. Empirically, this allows the network to learn disentangled factors of variation, corresponding to different subparts of the latent code. Meanwhile, because it is essentially a single-layer flat model, our VLAE does not exhibit the problems we have identified with traditional hierarchical VAE described in Section~\ref{sec:limitations}.

Ladder Variational Autoencoders (LVAE) on the other hand, utilize the ladder architecture from the inference/encoding side; its generative model is a standard HVAE.
%, so it is ``the ladder approach to variational autoencoders''. 
While the ladder inference network performs better than the one used in the original HVAE, ladder variational autoencoders still suffer from the problems we discussed in Section \ref{sec:limitations}. The difference is between our model (VLAE) and LVAE is illustrated in Figure~\ref{fig:vael}

An additional advantage over ladder variational autoencoders (and more generally HVAEs) is that our definition of the generative network Equ.(\ref{eq:vlae_generator}) allows us to select a much richer family of generative models $p$. Because for HVAE the $\mathcal{L}_{ELBO}$ optimization requires the evaluation of $\log p(\zv_\ell|\zv_{\ell+1})$ shown in Equ.(\ref{equ:hvae_criteria}), a reparameterized HVAE inject noise into the network in a way that corresponds to a conditional distribution with a tractable log-likelihood. For example, a HVAE can inject noise $\epsv_\ell$ by
\beq
\zv_\ell = \muv_\ell(\zv_{\ell+1}) + \sigmav_\ell(\zv_{\ell+1}) \odot \epsv_\ell \label{equ:hvae_reparameterize}
\eeq
only because this corresponds to Gaussian conditional distributions $p(\zv_l|\zv_{l+1})$. In comparison, for VLAE we only require evaluation of $\log p(\xv|\zv_1, \cdots, \zv_L)$, so except for the bottom layer $\rv$ we can combine noise by any arbitrary black box function $\fv_\ell$. %This coincides with the recently proposed implicit generative models. 
%\stefano{discuss implicit vs explicit models?} \shengjia{I think the models are slightly different because we do compute likelihoods so we are not implicit.}
%\stefano{this does not make sense. $f$ is deterministic, so it does not inject noise?}
%Note that due to the difference in architecture we require a change of notation here where $\zv_\ell$ is the injected noise for VLAE, as opposed to $\epsv_\ell$ for HVAE. 

%\stefano{reference figure 4 somewhere}

\section{Experiments}
\begin{figure*}
\centering
\includegraphics[trim=100 200 100 200, clip, width=\linewidth]{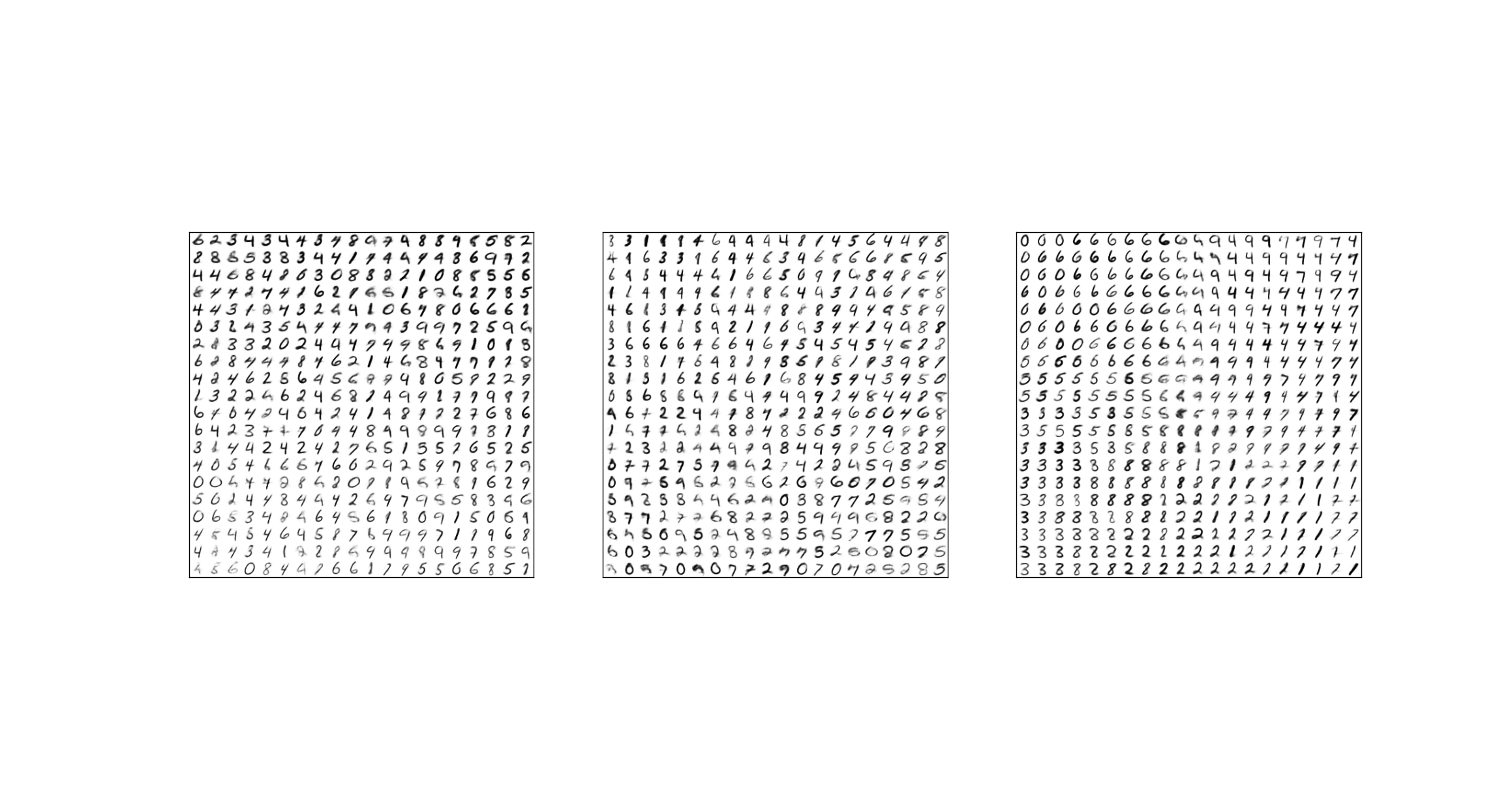}
\caption{VLAE on MNIST. Generated digits obtained by systematically exploring the 2D latent code from one layer, and randomly sampling from other layers. %Left: $\epsilon_0$, Middle: $\epsilon_1$, Right: $\epsilon_2$. 
\textbf{Left panel:} The first (bottom) layer encodes stroke width, \textbf{Center panel:} the second layer encodes digit width and tilt, \textbf{Right panel:} the third layer encodes (mostly) digit identity. Note that the samples are not of state-of-the-art quality only because of the restricted 2-dimensional latent code used to enable visualization. }
\label{fig:mnist_ladder_separate_feature}
\end{figure*}

\begin{figure*}
\centering
\includegraphics[trim=100 20 100 20, clip, width=\linewidth]{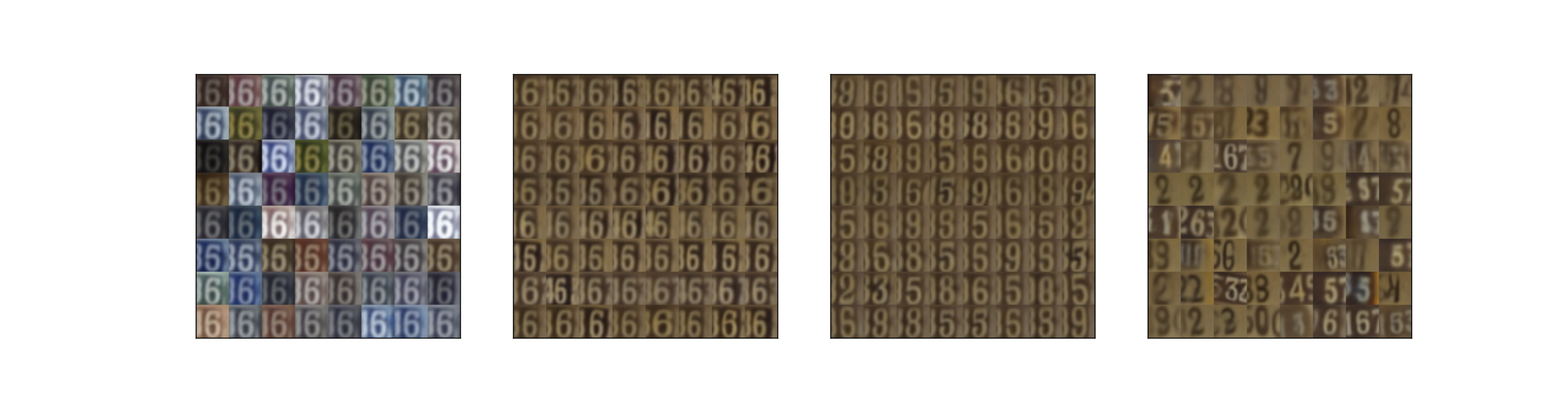}
\caption{VLAE on SVHN. Each sub-figure corresponds to images generated when fixing latent code on all layers except for one, which we randomly sample from the prior distribution. From left to right the random sampled layer go from bottom layer to top layer. \textbf{Left panel:} The bottom layer represents color schemes; \textbf{Center-left panel:} the second layer represents shape variations of the same digit; \textbf{Center-right panel:} the third layer represents digit identity (interestingly these digits have similar style although having different identities); \textbf{Right panel:} the top layer represents the general structure of the image.}
\label{fig:ladder_vae_svhn2}
\end{figure*}

\begin{figure*}
\centering
\includegraphics[trim=100 20 100 20, clip, width=\linewidth]{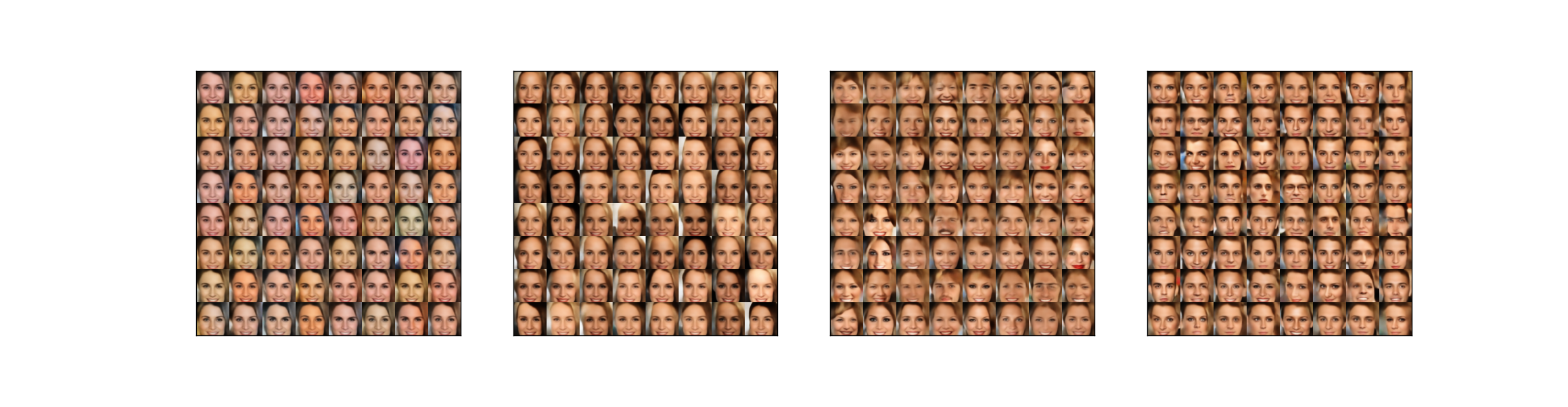}
\caption{VLAE on CelebA. Each sub-figure corresponds to images generated when fixing latent code on all layers except for one, which we randomly sample from the prior distribution.  From left to right the random sampled layer go from bottom layer to top layer. \textbf{Left panel:} The bottom layer represents ambient color; \textbf{Center-left panel:} the second bottom layer represents skin and hair color; \textbf{Center-right panel:} the second top layer represents face identity; \textbf{Right panel:} the top layer presents pose and general structure.}
\label{fig:ladder_vae_celebA2}
\end{figure*}

We train VLAE over several datasets and visualize the semantic meaning of the latent code. \footnote{Code is available at https://github.com/ShengjiaZhao/Variational-Ladder-Autoencoder} According to our assumptions, complex, high-level information will be learned by latent codes at higher layers, whereas simple, low-level features will be represented by lower layers. 

In Figure \ref{fig:mnist_ladder_separate_feature}, we visualize generation results from MNIST, where the model is a 3-layer VLAE with 2 dimensional latent code ($\zv$) at each layer. The visualizations are generated by systematically exploring the 2D latent code for one layer, while randomly sampling other layers. From the visualization, we see that the three layers encode stroke width, digit width and tilt and digit identity respectively. Remarkably, the semantic meaning of a particular latent code is stable with respect to the sampled latent codes from other layers. For example, in the second layer, the left side represents narrow digits whereas the right side represents wide digits. Sampling latent codes at other layers will control the digit identity, but have no influence over the width. This is interesting given that width is actually correlated with the digit identity; for example, digit 1 is typically thin while digit 0 is mostly wide. Therefore, the model will generate more zeros than ones if the latent code at the second layer corresponds to a wide digit, as displayed in the visualization.

Next we evaluate VLAE on the Street View House Number (SVHN, \citet{netzer2011reading}) dataset, where it is significantly more challenging to learn interpretable representations since it is relatively noisy, containing certain digits which do not appear in the center. However, as is shown in Figure \ref{fig:ladder_vae_svhn2}, our model is able to learn highly disentangled features through a 4-layer ladder, which includes color, digit shape, digit context, and general structure. These features are highly disentangled: since the latent code at the bottom layer controls color, modifying the code from other three layers while keeping the bottom layer fixed will generate a set of image which have the same tone in general. Moreover, the latent code learned at the top layer is the most complex one, which captures rich variations lower layers cannot accurately represent. 

Finally, we display compelling results from another challenging dataset, CelebA ~\citep{liu2015faceattributes}, which includes 200,000 celebrity images. These images are highly varied in terms of environment and facial expressions. We visualize the generation results in Figure \ref{fig:ladder_vae_celebA2}. As in the SVHN model, the latent code at the bottom layer learns the ambient color of the environment while keeping the personal details intact. Controlling other latent codes will change the other details of the individual, such as skin color, hair color, identity, pose (azimuth); more complicated features are placed at higher levels of the hierarchy.

\section{Discussions}
Training hierarchical deep generative models is a very challenging task, and there are two main successful families of methods. One family defines the destruction and reconstruction of data using a pre-defined process. Among them, LapGANs \citep{denton2015deep} define the process as repeatedly downsampling, and Diffusion Nets \citep{sohl2015deep} defines a forward Markov chain that coverts a complex data distribution to a simple, tractable one. Without having to perform inference, this makes training much easier, but it does not provide latent variables for other downstream tasks (unsupervised learning). 

Another line of work focuses on learning a hierarchy of latent variables by stacking single layer models on top of each other. Many models also use more flexible inference techniques to improve performance ~\citep{sonderby2016ladder, dinh2014nice, salimans2015markov, rezende2015variational, li2016learning, vae_autoregressive_flow2016}. However we show that there are limitations to stacked VAEs. %In particular Ladder Variational Autoencoders \citep{sonderby2016ladder} introduce a top-down refinement module in addition to the original bottom-up inference network, where it takes a bottom-up deterministic pass to evaluate state variables and then a top-down stochastic pass to obtain latent variables. Similar ideas are also proposed in \citep{bachman2016architecture}. %The Matryoshka Network \citep{bachman2016architecture} proposes the use of shortcut and residual connections to transform the generative architecture into what is essentially a special type of one-layer autoregressive model ~\citep{larochelle2011neural}.

Our work distinguishes itself from prior work by explicitly discussing the purpose of learning such models: the advantage of learning a hierarchy is not in better representation efficiency, or better samples, but rather in the introduction of structure in the features, such as hierarchy or disentanglement. This motivates our method, VLAE, which justifies our intuition that a reasonable network structure can be, by itself, highly effective at learning structured (disentangled) representations. Contrary to previous efforts on hierarchical models, we do not stack VAEs on top of each other, instead we use a ``flat'' approach. This can be applied in combination with the stacking approach. 

%instead of being only hierarchical in architecture, VLAE unlocks the potential to extract structured information from the latent variables through the use of a hierarchical deep generative model.  

The results displayed in the experiments resemble those obtained with InfoGAN \citep{chen2016infogan}; both frameworks learn disentangled representations from the data in an unsupervised manner. The InfoGAN objective, however, explicitly maximizes the mutual information between the latent variables and the observation; whereas in VLAE, this is achieved through the reconstruction error objective which encourages the use of latent codes. Furthermore we are able to explicitly disentangle features with different level of abstractness.  %On the other hand, the KL divergence term serves as a regularization term to approximately reduce that mutual information. 
% To see this, we show that the mutual information between $\xv$ and $\zv$ can be written as 
% \beq
% \mb{I}(\zv, \xv) = \bb{E}_{p(\zv \lvert \xv)} [\log p(\zv \lvert \xv) - \log p(\zv)]
% \eeq
% This is identical to $\bb{KL}(q(\zv \lvert \xv) \lVert p(\zv))$ if we replace $p(\zv \lvert \xv)$ with $q(\zv \lvert \xv)$, which happens to be exactly $p(\zv \lvert \xv)$ if we assume an inference network that is flexible enough. Therefore, larger KL loss for a particular $\zv_\ell$ may indicate that this latent code has higher mutual information with the data, which is what we observed in the experiments. 
%Instead of explicitly telling each code to maximize certain information according to a predefined prior, 
%Combined we implicitly encourage the latent variables to automatically capture the right amount of information.

\section{Conclusions}
In this paper, we discussed the potential practical value of learning a hierarchical generative model over a non-hierarchical one. We show that little can be gained in terms of representation efficiency or sample quality. We further show that traditional HVAE models have trouble learning structured features. Based on these insights, we consider an alternative to learning structured features by leveraging the expressive power of a neural network. Empirical results show that we can learn highly disentangled features.

One limitation of VLAE is the inability to learn structures other than hierarchical disentanglement.
Future work should consider more principled ways of designing architectures that allow for learning features with more complex structures.

\section{Acknowledgement}
This research was supported by Intel Corporation, NSF (\#1649208) and Future of Life Institute (\#2016-158687).

\bibliographystyle{icml2017}
\begin{small}
%\bibliography{ref}
\bibliography{bib}
\end{small} 

\newpage
\appendix

\section{Additional Results}
\begin{prop}
\label{prop:hvae_redundency}
$\mathcal{L}_{ELBO}$ for HVAE in Eq.(\ref{equ:hvae_criteria}) is optimized when $\mathcal{L}_{ELBO} = -H(p_{data}(x))$. If $\mathcal{L}_{ELBO}$ is optimized the following Gibbs sampling chain converges to $p_{data}(\xv)$ if it is ergodic
\begin{align*}
\zv^{(t)} &\sim q(\zv|\xv^{(t)}) \\ 
\xv^{(t+1)} &\sim p(\xv|\zv^{(t)}) \numberthis \label{equ:gibbs_chain_autoregressive}
\end{align*}
\end{prop}

\begin{proof}[Proof of Proposition~\ref{prop:hvae_redundency}]
As in the proof of Proposition~\ref{prop:hvae_redundency_markov} when $\mathcal{L}_{ELBO}$ is optimized, $q(\zv|\xv) = p(\zv|\xv)$. Because the following Gibbs chain converges to $p_{data}(\xv)$
\begin{align*}
\zv^{(t)} &\sim q(\zv|\xv^{(t)}) \\ 
\xv^{(t+1)} &\sim q(\xv|\zv^{(t)}) 
\end{align*}
We can replace 
%the possibly intractable 
$q(\xv|\zv^{(t)})$ with $p(\xv|\zv^{(t)})$ and the chain still converges to $p_{data}(\xv)$.
\end{proof}

\section{Experimental Details}

\subsection{Gaussian HVAE}
\textbf{Architecture: }
For $l=1, 2$
\[ \zv_l \sim \mathcal{N}(W_1 \fv_l(\zv_{l+1}), sigm(W_2 \fv_l(\zv_{l+1}))^2)\]
where $W_1, W_2$ are trainable linear transformation matrices, and $sigm$ is sigmoid activation function. $\fv_l$ is a two layer dense network. For $l=0$, we let
\[ \xv \sim \mathcal{N}(\fv_0(\zv_1), \sigma^2 I) \]
where $\sigma$ is a hyper-parameter that can be specified apriori or trained. $\fv_0$ is a two layer convolutional network with $1/2$ stride for spatial up-sampling. For inference we use the same architecture as the generator.

\textbf{Learning: }During training we use the Adam \citep{adam_optimization2014} optimizer with learning rate $10^{-4}$. We also anneal the scale the KL-regularization from $0$ to $1$ to encourage use of latent feature during early stages of training. 

\subsection{VLAE}
For VLAE, we use varying layers of convolution depending on size of input image. However, for the ladder connections we do not use convolution. Because of our argument in introduction and Figure~\ref{fig:unreasonable_face}, generative models do not benefit from convolutional latent features. Therefore we always flatten convolutional layers and apply linear transformation to reduce dimension for each ladder connection. For implementation details please refer to our code. 
\end{document}